\documentclass{article}
\usepackage{times,dsfont}
\usepackage{amsmath}
\usepackage{amsfonts}
\usepackage{amssymb}
\usepackage{amsthm}
\usepackage{subfigure}
\usepackage{epsfig}
\usepackage{color}
\usepackage{enumerate}
\usepackage{psfrag, graphics}

\usepackage{algorithm}
\usepackage{algorithmic}
\usepackage{natbib}
\usepackage{hyperref}
\usepackage[accepted]{icml2011}

\newtheorem{lemma}{Lemma}
\newtheorem{theorem}{Theorem}
\newtheorem{corollary}{Corollary}
\newtheorem{definition}{Definition}

\newcommand{\dist}{\mathcal{F}}
\newcommand{\X}{\mathcal{X}}

\newcommand{\Xspl}{\mathbf{X}}

\newcommand{\C}{\mathcal{C}}
\newcommand{\G}{\widetilde{G}}

\newcommand{\Sp}{\mathcal{S}}

\newcommand{\real}{\mathbb{R}}

\newcommand{\norm}[1]{\left\|#1\right\|}
\newcommand{\abs}[1]{\left|#1\right|}
\newcommand{\paren}[1]{\left(#1\right)}
\newcommand{\pr}[1]{\mathbb{P}\left(#1\right)}

\newcommand{\braces}[1]{\left\{#1\right\}}

\DeclareMathOperator*{\volume}{vol}

\newcommand{\vol}[1]{\volume\left(#1\right)}

\icmltitlerunning{Pruning nearest neighbor cluster trees}
\begin{document}
\twocolumn[
\icmltitle{Pruning nearest neighbor cluster trees}
                 % \\ International Conference on Machine Learning (ICML 2011)}

% unless you've provided the [accepted] option to the icml2011
\icmlauthor{Samory Kpotufe}{samory@tuebingen.mpg.de}
%\icmladdress{Max Planck Institute for Intelligent Systems}
\icmlauthor{Ulrike von Luxburg}{ulrike.luxburg@tuebingen.mpg.de}
\icmladdress{Max Planck Institute for Intelligent Systems, Tuebingen, Germany}

\icmlkeywords{boring formatting information, machine learning, ICML}
\vskip 0.3in
] 

\begin{abstract}
% Let $G_n$ be a (mutual) $k$-NN graph over a finite sample from a distribution $\dist$ with density $f$. Consider the process of 
% removing nodes from $G_n$ in decreasing order 
% of their $k$-NN radius. This process builds a tree of nested subgraphs of $G_n$, where each subgraph is a tree level and its
% connected components (CCs) are nodes of the tree. 
% 
% We show that, with probability tending to $1$ (consistency), this tree approximates the cluster tree of $\dist$, i.e. the nested hierarchy
% formed by the CCs of the level sets of $f$. 
% We then provide a simple pruning procedure with strong practical guarantees not offered in previous work: 
% \begin{enumerate}[(a)]
%  \item The pruning maintains consistency.
% \item With high probability clusters remaining at any level of the pruned tree correspond to clusters of the underlying density.
%\end{enumerate}
Nearest neighbor ($k$-NN) graphs are widely used in machine learning and data mining applications,
and our aim is to better understand what they reveal about the cluster structure of the unknown underlying 
distribution of points. Moreover, is it possible to identify spurious structures that might arise due to sampling 
variability?

Our first contribution is a statistical analysis that reveals how certain subgraphs of a $k$-NN graph 
form a consistent estimator of the cluster tree of the underlying distribution of points.
Our second and perhaps most important contribution is the following finite sample guarantee. 
We carefully work out the tradeoff between aggressive and conservative pruning and are able 
to guarantee the removal of all spurious cluster structures at all levels of the tree while at the same time guaranteeing 
the recovery of salient clusters. This is the first such finite sample result in the context of clustering.
% in providing the first 
% statistical guarantees in the context of clustering for a concrete approach to pruning all spurious 
% cluster structures, while still recovering salient clusters.

\end{abstract}

\section{Introduction}
% Nearest neighbor ($k$-NN) methods are among the oldest and most popular approaches in Statistics and Machine Learning, 
% used in such diverse tasks as classification, regression, density estimation and clustering \cite{ml:dhs_pattern_classification}. 
% Interestingly, despite the large amount of literature on the subject of $k$-NN, there is still much to understand. 
In this work, we consider the nearest neighbor ($k$-NN) graph where each sample point is 
linked to its nearest neighbors. These graphs are widely used in machine learning and data mining applications, and interestingly there is 
still much to understand about their expressiveness. In particular we would like to 
better understand what such a graph on a finite sample of points might reveal about the cluster structure of 
the underlying distribution of points. More importantly 
we are interested in whether one can identify spurious structures that are artifacts of sampling variability, i.e. 
spurious structures that are not representative of the true cluster structure of the distribution.
\begin{figure}
 \centering
%\resizebox{5cm}{!}{\input{clusterTree2.pstex_t}}
\includegraphics[width=4.5cm]{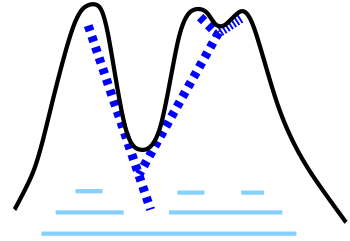}
\caption{A density $f$ (black line) and its cluster tree (dashed). The CCs of 3 level sets are shown in lighter color at the bottom.}
\label{fig:cluster tree}
\end{figure}

Our first contribution is in exposing more of the richness of $k$-NN graphs. 
Let $G_n$ be a $k$-NN graph over an $n$-sample from a distribution $\dist$ with density $f$. 
Previous work \cite{MHL105} has shown that the connected components (CC) of a given level set of $f$ 
can be approximated by the CCs of some subgraph of $G_n$, provided the level set satisfies certain 
boundary conditions. However it remained unclear whether or when all level sets of $f$ might satisfy these 
conditions, in other words, whether the CCs of any level set can be recovered. We 
show under mild assumptions on $f$ that CCs of any level set can be recovered by subgraphs of $G_n$ 
for $n$ sufficiently large. Interestingly, these subgraphs are obtained in a rather simple way: just 
remove points from the graph in decreasing order of their $k$-NN radius (distance to the $k$'th 
nearest neighbor), and we obtain a nested hierarchy of subgraphs which approximates the \emph{cluster tree} of $\dist$, 
i.e. the nested hierarchy formed by the level sets of $f$ (see Figure \ref{fig:cluster tree}, also Section \ref{sec:clusterTree}). 
% We treat the $k$-NN graph generally, considering various forms in which it appears in practice (see Definition \ref{def:knn}).
% Our results are presented in a way to give a clean sense of both the finite sample behavior of these graphs
% and their asymptotic behavior as $n$ grows.

Our second, and perhaps more important contribution is in providing the first concrete approach in the context of clustering
 that guarantees the pruning of all spurious cluster structures at any tree level. We carefully work out the tradeoff 
between pruning ``aggressively'' (and potentially removing important clusters) and pruning ``conservatively'' 
(with the risk of keeping spurious clusters) and derive tuning settings that require no knowledge of the underlying 
distribution beyond an upper bound on $f$. We can thus guarantee in a finite sample setting that 
(a) all clusters remaining at any level of the pruned tree correspond to CCs of some level set of $f$, 
i.e. all spurious clusters are pruned away, and (b) salient clusters are still discovered, where 
the degree of \emph{saliency} depends on the sample size $n$. We can show furthermore that
 the pruned tree remains a consistent estimator of the underlying cluster tree, i.e. the CCs of 
any level set of $f$ are recovered for sufficiently large $n$. 
Interestingly, the pruning procedure is not tied to the $k$-NN method, but is based on a simple intuition that can be applied 
to other cluster tree methods (see Section \ref{sec:alg}).

Our results rely on a central ``connectedness'' lemma (Section \ref{sec:connectedness}) that identifies which CCs of $f$ 
remain connected in the empirical tree. This is done by analizing the way in which $k$-NN radii vary along a path in a dense region.
% Moreover, we address the following problem of high practical value encountered in clustering (in our case clusters 
% at a particular level are the CCs of subgraphs of $G_n$). We might get spurious clusters due to sampling variability and, 
% not knowing the distribution $\dist$, these are hard to identify. A natural idea present in various 
% forms in the literature is to remove \emph{small} clusters, though it remains unclear whether a given definition 
% of \emph{small} is appropriate for the unknown $\dist$ we sampled from. In particular we know of no pruning procedure in 
% the literature that guarantees the removal of all spurious clusters while still recovering salient clusters. We derive
% such a pruning procedure. We provide a simple tuning parameter that controls how aggressively we 
%prune the empirical cluster tree. 
% We show how to set this parameter and still guarantee that the pruned tree remains a consistent 
% estimator of the underlying cluster tree. Furthermore, for finite samples, we guarantee that, with high probability, 
% clusters remaining at any level of the pruned tree correspond to separate CCs of some level set of $f$.

% We treat the $k$-NN graph generally, considering various forms in which it appears in practice (see Definition \ref{def:knn}).
% We make minimal assumptions on $\dist$ and present our results in a way to give a clean sense of 
% both the finite sample behavior of these graphs (and their pruned version), and their asymptotic behavior as $n$ grows.

\subsection{Related work}
Recovering the cluster tree of the underlying density is a clean formalism of 
hierarchical clustering proposed in 1981 by J. A. Hartigan \cite{H106}. 
Hartigan showed in the same seminal paper that the single-linkage algorithm is a consistent estimator 
of the cluster tree for densities on $\real$. For $\real^d, d>1$ it is known that the empirical cluster tree of 
a consistent density estimate is a consistent estimator of the underlying cluster tree (see e.g. \cite{WT108}), unfortunately
 there is no known algorithm for computing this empirical tree. Nonetheless, the idea has led to the development of 
interesting heuristics based on first estimating density, then approximating the cluster tree of the density estimate 
in high dimension \cite{WT108, SN109}.

Many other related work such as \cite{RV110, SSN111, MHL105, RW112} 
consider the task of recovering the CCs of a single level set, the closest to the present work being \cite{MHL105}
which uses a $k$-NN graph for level set estimation. As previously discussed, level set estimation however never led 
to a consistent estimator of the cluster tree, since these results typically impose technical requirements on the level set being 
recovered but do not work out how or when these requirements might be satisfied by all level sets of a distribution.

A recent insightful paper of \citet{CD101} presents the first provably consistent algorithm 
for estimating the cluster tree. At each level of the empirical cluster tree, 
they retain only those samples whose $k$-NN radii are below a scale parameter $r$ which indexes the level;  
CCs at this level are then discovered by building an $r$-neighborhood graph on the retained samples.
%, i.e. points are connected to those points that are within distance $r$ of them. 
This is similar to an earlier generalization of single-linkage by \citet{W:57} which however was given without a convergence analysis.
The $k$-NN tree studied here differs in that, at an equivalent level $r$, points are connected to the subset of their $k$-nearest neighbors retained at that level. 
One practical appeal of our method is its simplicity: we need only remove points from an initial $k$-NN graph to obtain the various levels of the empirical cluster tree.
%; in contrast, that of \cite{CD101} builds a new $r$-neighborhood graph at each level $r$.

\cite{CD101} provides finite sample results for a particular setting of $k\approx \log n$. In contrast our finite sample results 
are given for a wide range of values of $k$, namely for $\log n \lesssim k \lesssim n^{1/O(d)}$. In both cases the 
finite sample results establish natural separation conditions under which the CCs of level sets are recovered 
(see Theorem \ref{theo:main}). The result of \cite{CD101} however allows the possibility that 
some empirical clusters are just artifacts of sampling variability. We provide a simple pruning procedure 
that ensures that clusters discovered empirically at any level correspond to true clusters at some level or the underlying
cluster tree. Note that this can be trivially guaranteed by returning a single cluster at all levels, so we additionally 
guarantee that the algorithm discovers salient modes of the density, where the saliency depends on empirical 
quantities (see Theorem \ref{theo:pruning}).

A recent archived paper \cite{RSNW107} also treats the problem of false clusters in cluster tree estimation, but 
the result is not algorithmic as they only consider the cluster tree of an empirical density estimate, and do not provide a way to compute this cluster tree.

There exist many pruning heuristics in the literature which typically consist
of removing \emph{small} clusters \cite{MHL105, SN109} using some form of thresholding. 
The difficulty with these approaches is in how to define \emph{small} without 
making strong assumptions on the unknown underlying distribution, or on the tree level 
being pruned (levels correspond to different resolutions or cluster sizes).
Moreover, even the assumption that spurious clusters must be small does not necessarily hold.
Consider for example a cluster made up of two large regions 
connected by a thin bridge of low mass; the two large regions can easily appear as two separate clusters in a finite sample.
Some more sophisticated methods such as \cite{SN113} do not rely on cluster size for pruning, instead 
 they return confidence values for the empirical clusters based on 
various notions of cluster stability; unfortunately they do not provide finite sample guarantees.
Our pruning guarantees the removal of all spurious clusters, large and small (see Figure \ref{fig:pruning}); we make no assumption on the shape of clusters 
beyond a smoothness assumption on the density; we provide a simple tuning parameter whose setting requires 
just an upper bound on the density.

%Their algorithm is simple and practical and they provide an insightful finite sample analysis which implies consistency. 
%We build upon some of their insights and in particular 
%we make use of their very natural formalism of separation between clusters (see Definition \ref{def:separation}). 

\section{Preliminaries}
Assume the finite dataset $\Xspl= \braces{X_i}_{i =1}^n$ is drawn i.i.d. from a distribution $\dist$ over $\real^d$ with density function $f$.

We start with some simple definitions related to $k$-NN operations. All balls, unless otherwise specified, denote closed balls in $\real^d$.
\begin{definition}[$k$-NN radii]
 For $x\in \X$, let $r_{k,n}(x)$ denote the radius of the smallest ball centered at $x$ containing $k$ points from
$\Xspl\setminus\{x\}$. Also, let $r_k(x)$ denote the radius of the smallest ball centered at $x$ of $\dist$-mass
$k/n$.
\end{definition}

\begin{definition}[$k$-NN and mutual $k$-NN graphs]
\label{def:knn}
The $k$-NN graph is that whose vertices are the points in $\Xspl$, and where $X_i$ is connected to $X_j$ iff $X_i \in B(X_j, \theta
r_k(X_j))$ or   $X_j \in B(X_i, \theta  r_k(X_i))$ for some $\theta > 0$. The mutual $k$-NN graph is that where $X_i$ is connected to
$X_j$ iff $X_i \in B(X_j, \theta  r_k(X_j))$ and $X_j \in B(X_i, \theta  r_k(X_i))$.
\end{definition}

\subsection{Cluster tree}
\label{sec:clusterTree}
\begin{definition}[Connectedness]
 We say $A\subset \real^d$ is connected if for every $x, x' \in A$ there exists a continuous $1-1$ function $P:[0, 1]\mapsto A$ where
$P(0) = x$ and $P(1) = x'$. $P$ is called a path in $A$ between $x$ and $x'$. 
%We sometimes abuse notation and denote the image of $P$ (the geometric path) by $P$, when clear in context.
\end{definition}

The cluster tree of $f$ will be denoted $\braces{G(\lambda)}_{\lambda>0}$, where $G(\lambda)$ are the CCs of 
the level set $\braces{x: f(x) \geq \lambda}$. Notice that $\braces{G(\lambda)}_{\lambda>0}$ forms a (infinite) tree hierarchy where for any
two components $A, A'$, either $A\cap A' = \emptyset$ or one is a descendant of the other, i.e $A\subset A'$ or $A'\subset A$.

\section{Algorithm}
\label{sec:alg}
\begin{definition}[$k$-NN density estimate]
Define the density estimate at $x\in \real^d$ as :
$$f_n(x)\doteq \frac{k}{n\cdot \vol{B(x, r_{k,n}(x))}} = \frac{k}{n\cdot  v_d  r_{k,n}^d(x)},$$
where $v_d$ is the volume of the unit ball in $\real^d$. 
\end{definition}

Let $G_n$ be the $k$-NN or mutual $k$-NN graph.  For $\lambda>0$ define
$G_n(\lambda)$ as the subgraph of $G_n$ containing only vertices in $\braces{X_i: f_n(X_i) \geq \lambda}$ 
and corresponding edges. The CCs of $\braces{G_n(\lambda)}_{\lambda>0}$ form a tree: 
let $A_n$ and $A_n'$ be two such CCs, either  $A_n\cap A_n' = \emptyset$ or one is a descendant of the 
other, i.e. $A_n$ is a subgraph of $A_n'$ or vice versa. 
To simplify notation, we let the set $\braces{G_n(\lambda)}_{\lambda>0}$ denote the empirical cluster tree 
before pruning. 

\subsection*{Pruning}
% The aim of the pruning is to remove spurious branches (equivalently spurious clusters) resulting from sampling variability. Essentially, 
% the procedure looks at the difference between the first time a branching occurs in the empirical cluster tree (levels
% $\lambda_1$ and $\lambda_2$ depicted below) and the levels of resulting leaves. If this difference is above a threshold
% $\tilde{\epsilon}$, the branching is kept, otherwise those branches are pruned. 
% 
% \begin{center}
%  \resizebox{8cm}{!}{\input{pruning.pstex_t}}
% \end{center}
% 
% The pruning procedure (Algorithm \ref{alg:pruning}) does not operate directly as described above. Instead, it consists of simple lookups from a level down to
% levels $\tilde{\epsilon}\geq 0$ away. The combinations of these lookups result in the mode of operation described above. 
The pruning procedure (Algorithm \ref{alg:pruning}) consists of simple lookups: it reconnects CCs at level $\lambda$ if they are part of the same 
CC at level $\lambda - \tilde{\epsilon}$ where the tuning parameter $\tilde{\epsilon}\geq 0$ controls how aggressively we prune. 
We show its behavior on a finite sample in Figure \ref{fig:pruning}.

The intuition behind the procedure is the following. Suppose $A_n, A'_n\subset \Xspl$ are disconnected 
at some level $\lambda$ in the empirical tree before pruning. However, they ought to be connected, i.e. 
their vertices belong to the same CC $A$ 
at the highest level where they are all contained in the underlying cluster tree. Then, 
key sample points from $A$ that would have kept them connected are missing at level $\lambda$ in the empirical tree. 
These key points have $f_n$ values lower than $\lambda$, but probably not much lower. 
% Points from a connected 
% set $A\subset \real^d$ might remain connected in the empirical tree at some level, but spurious branchings might 
% occur when subsets of $A\cap \Xspl$ become unconnected at higher levels due to the absence of key sample points
% that would have kept them otherwise connected. These points, depicted in gray below, 
% have lower $f_n$ values. By looking down to a level with lower $f_n$ values 
% %(the lookup distance corresponds to the error in density estimation) 
% we find all of $A\cap\Xspl$ to be connected and thus detect such a situation.
% \begin{center}
%  \includegraphics[width=5cm]{pruningIntuition.eps}
% \end{center}
By looking down to a lower level near $\lambda$ 
we find that $A_n, A_n'$ are connected and thus detect the situation.
Notice that this intuition is not tied to the $k$-NN cluster tree but can be applied to any other cluster tree procedure. All 
that is required is that all points from $A$ (as discussed above) be connected at some level in the tree close to $\lambda$.

\begin{algorithm}[h]
\begin{small}
   \caption{Prune $G_n(\lambda)$}
   \label{alg:pruning}
\begin{algorithmic}
\STATE Given: tuning parameter $\tilde{\epsilon}\geq 0$, same for all levels.
\STATE$\G_n(\lambda) \leftarrow G_n(\lambda)$.
\IF {$\lambda > \tilde{\epsilon}$}
     \STATE Connect components $A_n, A_n'$ of $\G_n(\lambda)$ if they are part of the same component 
of $G_n(\lambda -\tilde{\epsilon})$.
\ELSE 
\STATE Connect all $\G_n(\lambda)$.% for $\lambda < \min_{i\in [n]} f_n(X_i) + \tilde{\epsilon}$. 
\ENDIF
\end{algorithmic}
\end{small}
\end{algorithm}
It is not hard to see that the CCs of the pruned subgraphs $\braces{\G_n(\lambda)}_{\lambda>0}$ still form a tree.
We will hence denote the pruned empirical tree by $\braces{\G_n(\lambda)}_{\lambda>0}$.

\begin{figure}
 \centering
\includegraphics[height = 4cm, width = 4.2cm]{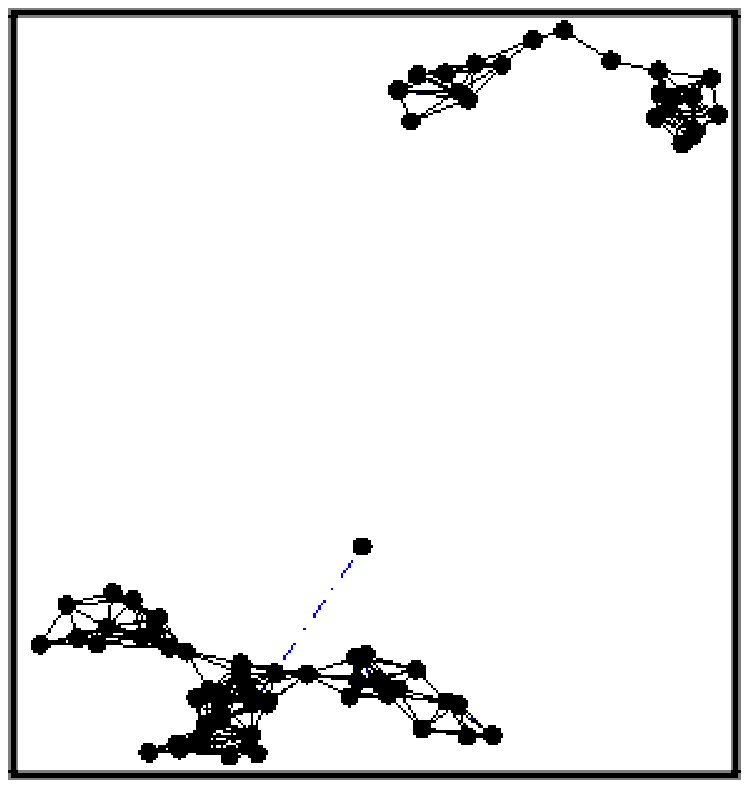}%figPruning.eps}
\hspace*{-1.0mm}
\includegraphics[height = 4cm, width = 4.cm]{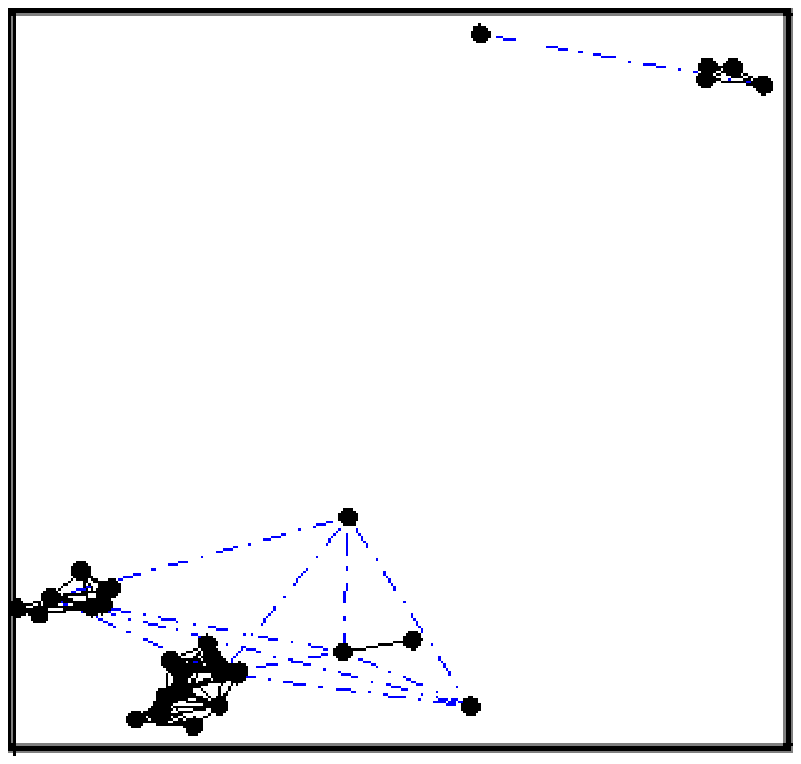}%figPruning2.eps}
\caption{Pruning at work: it reconnects CCs independent of size. The dashed lines are reconnection edges from pruning. 
Shown are two levels of the $k$-NN tree of a 500-sample from the 2-modes mixture 
$0.5\mathcal{N}([0 ,0], I_2) + 0.5 \mathcal{N}([1, 4], I_2)$. 
%a mixture of two normals with equal weights and means 4 apart. 
Here $k=12$, $\theta =1$, $\tilde{\epsilon}=F/\sqrt{k}$ where $F = 2.73$ is the maximum $f_n$ value.
From left to right, level $\lambda=0.9$ has 72 points, and level $\lambda=1.3$ has 33.}
\label{fig:pruning}
\end{figure}

\section{Results Overview}
We make the following assumptions on the density $f$.
\begin{enumerate}
 \item [(A.1)] $\exists F > 0$,  $\sup_{x\in \real^d}f(x) \leq F$.
\item  [(A.2)] $f$ is Hoelder-continuous, i.e. there exists $L, \alpha>0$ such that for all $x,
x'\in \real^d$, $$\abs{f(x) - f(x')}\leq L \norm{x-x'}^\alpha.$$ 
\end{enumerate}
% We note that the smoothness assumptions can be relaxed by requiring the density to only be smooth on the connected components of the
% support $\X \subset \real^d$. However this complicates the analysis (in handling support boundaries) and needlessly obscures the
% presentation.

%\subsection*{Main Results}
Theorem \ref{theo:main} below is a finite sample result that establishes conditions under which samples from a connected subset of $\real^d$
remain connected in the empirical cluster tree, and samples from two disconnected subsets of $\real^d$ remain disconnected even after
pruning. Essentially, for $k$ sufficiently large, points from connected subsets $A$ remain connected below some level. 
Also, provided $k$ is not too large, disjoint subsets $A$ and $A'$ which are separated by a large enough region 
of low density (relative to $n$, $k$ and $\tilde{\epsilon}$),  remain disconnected above some level.

We require the following two definitions.
\begin{definition}[Envelope of $A\subset \real^d$]
\label{def:envelope}
Let $A\subset \real^d$ and for $r>0$, define: $A_{+r} \doteq \braces{y: \exists x\in A, y \in B(x, r)}.$
% \begin{align*}
%  %A_{-r} \doteq \braces{x\in A: B(x, r) \subset A}, \text{ and } 
%  A_{+r} \doteq \braces{y: \exists x\in A, y \in B(x, r)}.
% \end{align*}
\end{definition}

\begin{definition}[$(\epsilon, r)$-separated sets ]
\label{def:separation}
 $A, A'\subset \real^d$ are $(\epsilon, r)$-separated if there exists a separating set $S$ such that every path in $\real^d$
 between $A$ and
$A'$ intersects $S$, and 
$$\sup_{x\in S_{+r}} f(x) \leq \inf_{x\in A\cup A'} f(x) - \epsilon.$$ 
\end{definition}

\begin{theorem}
 Suppose $f$ satisfies (A.1) and (A.2). Let $G_n$ be the $k$-NN or mutual $k$-NN graph.  Let $\delta>0$ and  
define $\epsilon_k \doteq11F\sqrt{\ln(2n/\delta)/k}$. 
There exist $C$ and $C' = C'(\dist)$ such that, for
\begin{align}
 &C  \paren{\max\braces{1, {\sqrt{2}}/{\theta}}}^d d \ln (n/\delta) \nonumber\\
&\leq k \leq  C'  \paren{F\sqrt{\ln (n/\delta)}}^{{2(\alpha +
d)}/{(3\alpha +d)}} n^{{2\alpha}/{(3\alpha + d)}} \label{eq:setting_of_k}
\end{align}
the following holds with probability at least $1- 3\delta$ simultaneously for subsets $A$ of $\real^d$. 
\begin{enumerate}[(a)]
 \item Let $A$ be a connected subset of $\real^d$, and let $\lambda \doteq \inf_{x\in A} f(x) > 2 \epsilon_k$. All points in $A\cap \Xspl$
belong to 
the same CC of $\G_n(\lambda - 2\epsilon_k)$.
\item Let $A$ and $A'$ be two disjoints subsets of $\real^d$, and define $\lambda = \inf_{x\in A\cup A'}f(x)$. 
Recall that $\tilde{\epsilon}\geq 0$ is the tuning parameter.
Suppose $A$ and $A'$ are $(\epsilon, r)$-separated for $\epsilon = 6\epsilon_k + 2\tilde{\epsilon}$ and 
$r=\frac{\theta}{2}\paren{{4k}/{v_d n \lambda}}^{1/d}$. Then $A\cap \Xspl$ and $A'\cap \Xspl$ are disconnected in $\G_n(\lambda - 2\epsilon_k)$. 
\end{enumerate}

\label{theo:main}
\end{theorem}

Theorem \ref{theo:main} above, although written in terms of $\G_n$, applies also to $G_n$ by just setting $\tilde{\epsilon} = 0$.
The theorem implies consistency of both pruned and unpruned $k$-NN trees under mild additional conditions. 
Some such conditions are illustrated in the corollary below. A nice practical aspect of the 
pruning procedure is that consistency is obtained for a wide range of settings of $\tilde{\epsilon}$ and $k$ as
functions of $n$.

\begin{corollary}[Consistency]
Suppose that $f$ satisfies (A.1) and (A.2) and that, in addition, $\dist$ is supported on a compact set, and for any $\lambda>0$, there are
finitely many components 
in $G(\lambda)$. Assume that, as $n \rightarrow \infty$, $\tilde{\epsilon} = \tilde{\epsilon}(n) \rightarrow 0$ and $k / \log n \rightarrow
0$ while $k=k(n)$ satisfies (\ref{eq:setting_of_k}).

For any $A\subset \real^d$, let $A_n$ denote the smallest component of 
$\braces{\G_n(\lambda)}_{\lambda>0}$ containing $A \cap \Xspl$. 
Fix $\lambda >0$. We have 
$\lim_{n \rightarrow \infty}\pr{\forall A, A' \in G(\lambda), \, A_n \text{ is disjoint from } A_n' } = 1$.
\label{cor:consistency}
\end{corollary}
\begin{proof}
 Let $A$ and $A'$ be separate components of $G(\lambda)$. 
The assumptions ensure that all paths between $A$ and $A'$ traverse a compact set $S$ satisfying $\lambda - \max_{x\in S} f(x) \doteq 
\epsilon_S >0$ (see Lemma 14 of \cite{CD101}).
Let $\epsilon = 6\epsilon_k + 2\tilde{\epsilon}$ and $r=\frac{\theta}{2}\paren{{4k}/{v_d n \lambda}}^{1/d}$.
By uniform continuity of $f$, there exists $N_1$ such that for $n>N_1$, $r$ is small enough so that $\lambda - \max_{x \in
S_{+r}}f(x) > \epsilon_S/2$.
Also, there exists $N_2 > N_1$ such that for $n>N_2$,  $\epsilon < \epsilon_S/2$, in other words 
$\sup_{x\in S_{+r}} f(x) \leq \lambda - \epsilon$.

Since $G_n(\lambda)$ is finite, there exists $N$ such that for $n>N$, all pairs $A, A'$ have a suitable $(\epsilon, r)$-separating set $S$.
Thus by Theorem \ref{theo:main}, for $n> N$, with probability at least $1-3\delta$, $\forall A, A' \in G(\lambda)$, $A\cap \Xspl$ and
$A'\cap \Xspl$ are fully contained in $\G_n(\lambda - 2\epsilon_k)$ 
and are disjoint. They are thus disjoint at any higher level, so $A_n$ and $A_n'$ are also disjoint. 

The above holds for all $\delta>0$, so the statement follows.
\end{proof}

While Theorem \ref{theo:main} establishes 
that a connected set $A$ remains connected below some level, it does not guarantee against parts of $A$ 
becoming disconnected at higher levels, creating spurious clusters.
Note that the removal of spurious clusters can be trivially guaranteed by just letting the parameter $\tilde{\epsilon}$ very large, but the ability of the algorithm to 
discover true clusters is necessarily affected. We are interested in how to set $\tilde{\epsilon}$ in order 
to guarantee the removal of spurious clusters while still recovering important ones.

Theorem \ref{theo:pruning} guarantees that,  by setting $\tilde{\epsilon}$ as $\Omega(\epsilon_k)$ 
(recall $\epsilon_k$ from Theorem \ref{theo:main}), separate 
CCs of the empirical cluster tree correspond to actual clusters of the (unknown) underlying distribution, i.e. 
all spurious clusters are removed.
The setting of $\tilde{\epsilon}$ only requires an upper-bound $F$ on the density $f$
\footnote{We might just use $\max_{i\in [n]}{f_n(X_i)}$ in practice, which in light of Lemma \ref{cor:knnbound} can be a
good surrogate for $F$ (see Figure \ref{fig:modes}).}.  
Note that, under such a setting, consistency is maintained per Corollary \ref{cor:consistency}, and 
in light of Theorem \ref{theo:main} (b), we can expect that interesting clusters are discovered. 
In particular the following salient modes of $f$ are discovered. 

\begin{definition}[$(\epsilon, r)$-salient mode]
\label{def:modes}
An $(\epsilon, r)$-salient mode is a leaf node $A$ of the cluster tree $\braces{G(\lambda)}_{\lambda>0}$ which 
has an ancestor $A_k\supset A$ (possibly $A$ itself) 
satisfying:
\begin{enumerate}[(\rm i)]
\item $A_k$ is the ancestor of a single leaf of $\braces{G(\lambda)}_{\lambda>0}$, namely $A$. 
 \item $A_k$ is large: $\exists x \in A_k, B(x, r_k(x)) \subset A_k$.
\item  $A_k$ is sufficiently separated from other components at its level: let $\lambda \doteq \inf_{x\in A_k} f(x)$;
 $A_k$ and $\paren{\braces{x: f(x)\geq \lambda} \setminus A_k}$
are $(\epsilon, r)$-separated.
\end{enumerate}

\end{definition}
Notice that, under 
the assumptions of Corollary \ref{cor:consistency}, every mode of $f$ is $(\epsilon, r)$-salient 
for sufficiently large $k$ and $1/\tilde{\epsilon}$.

\begin{theorem}[Pruning guarantees]
\label{theo:pruning}
 Let $\delta>0$. 
Under the assumptions of Theorem \ref{theo:main}, the following holds with probability at least $1- 3\delta$.
\begin{enumerate}[(a)]
 \item Suppose the tuning parameter $\tilde{\epsilon} \geq 3 \epsilon_k$. Consider two disjoint 
CCs  $A_n$ and $A_n'$ at the same level in $\braces{\G_n(\lambda)}_{\lambda>0}$.
Let  $V$ be the union of vertices of $A_n$ and $A_n'$, and define $\lambda \doteq \inf_{x\in V} f(x)$. 
The vertices of $A_n$ and those of $A_n'$ are in separate CCs of $G(\lambda)$.
\item  Let $\epsilon = 6\epsilon_k + 2\tilde{\epsilon}$ and 
$r=\frac{\theta}{2}\paren{{4k}/{v_d n \lambda}}^{1/d}$. There exists a $1-1$ map 
from the set of $(\epsilon, r)$-salient modes to the leaves of the empirical tree
$\braces{\G_n(\lambda)}_{\lambda >0}$.
\end{enumerate}

\end{theorem}
\begin{figure}
 \centering
\includegraphics[height=3.9cm]{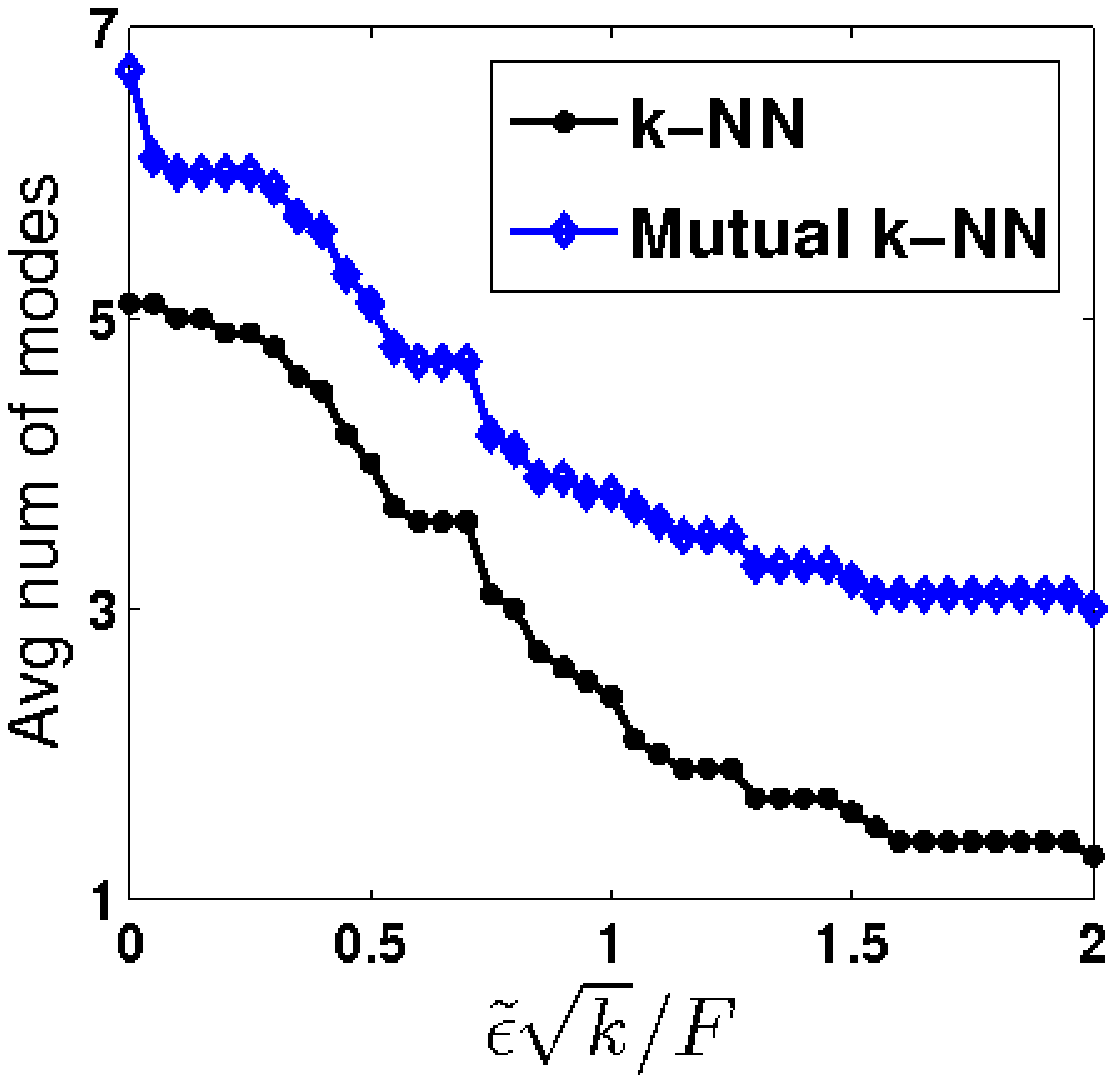}
\includegraphics[height=3.8cm]{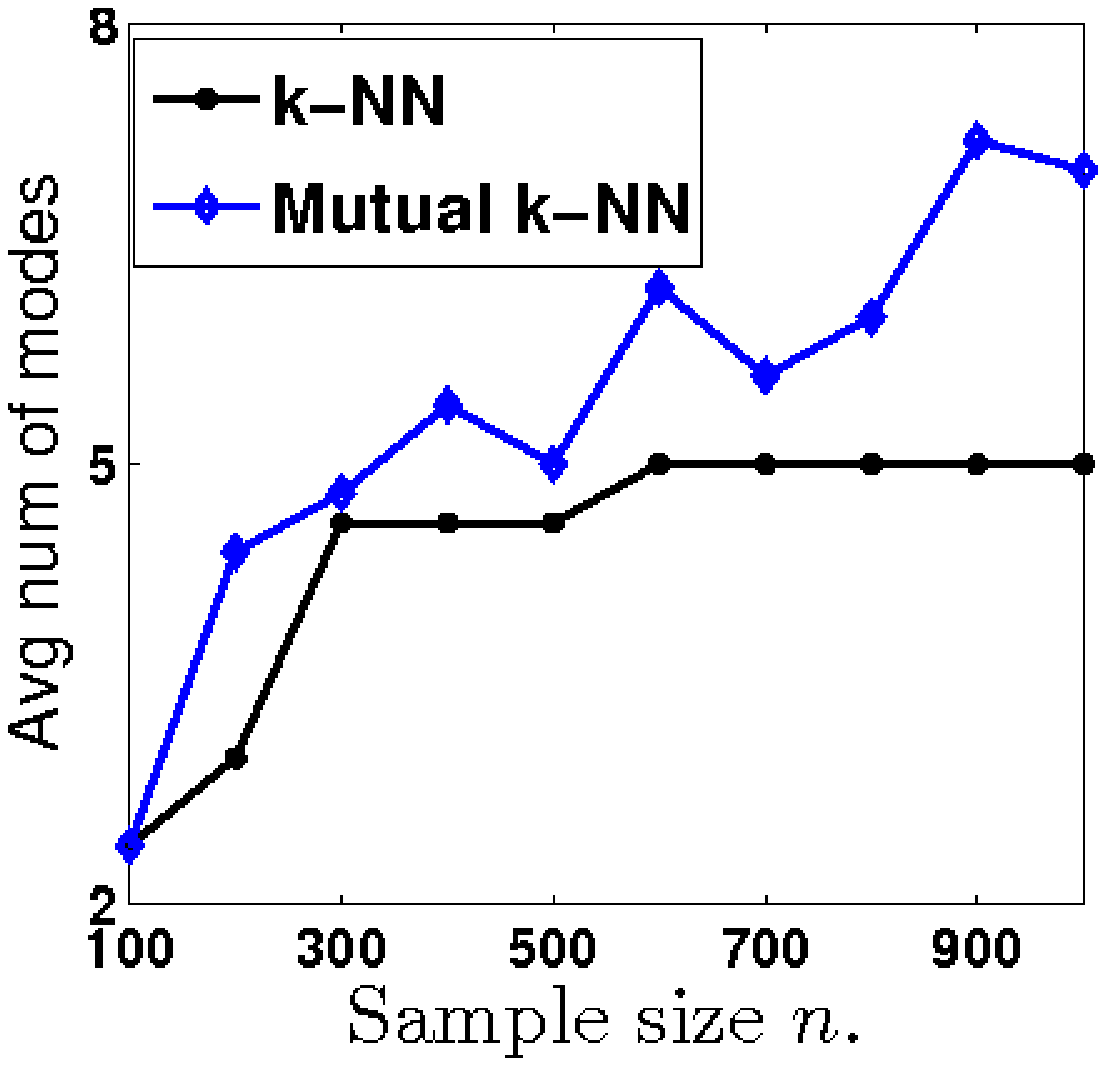}
\caption{(LEFT). Number of modes (leaves of the empirical tree) as we increase $\tilde{\epsilon}$ from 0. The trees are built on 500-samples 
(results are averaged over ten such 500-samples) from the 5-modes mixture $\sum_{i=1}^5 0.2\mathcal{N}(2\sqrt{d}e_i, I_d)$, 
$d=7$. Here $k=(\log n)^{1.5}$,  $\theta =1$, and 
$F$ is the maximum $f_n$ value over the 10 samples. 
The mutual $k$-NN tree being more sparse is rather brittle and requires more pruning. 
(RIGHT) We fix $\tilde{\epsilon} = F/4\sqrt{k}$, $k=(\log n)^{1.5}$, as we increase $n$.
Results are averaged over 10 n-samples for each $n$, and $F$ is again the max $f_n$ value 
over the 10 samples for each $n$. 
The $k$-NN tree quickly asymptotes at 5 modes. The mutual $k$-NN being more brittle, we're underpruning
for $n> 500$, i.e. $\tilde{\epsilon}$ is too small; thus for 
these settings we would require larger $n$ to obtain the correct number of modes. }
\label{fig:modes}
\end{figure}

The behavior of both the $k$-NN and mutual $k$-NN tree, as guaranteed in Theorem \ref{theo:pruning},
 is illustrated in Figure \ref{fig:modes}.

\section{Analysis}
Theorem \ref{theo:main} follows from lemmas \ref{lem:separation} and \ref{lem:connectedness} below. These two lemmas depend on 
the events described by lemmas \ref{cor:knnbound}, \ref{cor:rk_rk_n} and \ref{lem:VC} which happen with a combined probability of at
least $1-3\delta$ for a confidence parameter $\delta>0$.

Theorem \ref{theo:pruning} follows from lemmas \ref{lem:modes} and \ref{lem:pruning} below. 
These two lemmas also depend on the events described by lemmas \ref{cor:knnbound}, \ref{cor:rk_rk_n} 
and \ref{lem:VC} which happen with a combined probability of at least $1-3\delta$.

\subsection{Maintaining Separation}
In this section we establish conditions under which points from two disconnected subsets of $\real^d$ remain disconnected in 
the empirical tree, even after pruning.

The following is an important lemma which establishes the estimation error of $f_n$ relative to $f$ on the sample $\Xspl$. 
%It will be used repeatedly throughout the analysis. 
Interestingly, although of independent interest, we could not find this 
sort of finite sample statement in the literature on
$k$-NN\footnote{There are however many asymptotic analyses of $k$-NN methods such as \cite{DW104}.}, at least not under our assumptions. 
The proof, presented as supplement in the appendix, is a bit involved and starts with some intuition from 
an asymptotic analysis of \cite{DW104} combined with a form of the Chernoff bound found in \cite{AV102}.

\begin{lemma}
\label{cor:knnbound}
 Suppose $f$ satisfies (A.1) and (A.2). There exists $C = C(\dist)$ such that for $\delta>0$, for
$\epsilon = 11F\sqrt {{\ln (2n/\delta) }/{k}}$ and
% \begin{align}
%  \epsilon = F\sqrt {\frac{\ln (2n/\delta) }{k}} \text{ and }
% \ln (2n/\delta) < k < \paren{\frac{v_d}{2^{d+2}}\cdot n}^{2\alpha/(3\alpha + d)}(2L)^{2d/(3\alpha + d)}\paren{ {F} \sqrt{\log
% (2n/\delta)}}^{2(d+ \alpha)/(3\alpha+ d)}, 
% \end{align}
\begin{align*}
&121\ln (2n/\delta) \\
&\leq k \leq C\paren{ {F} \sqrt{\ln(2n/\delta)}}^{2(\alpha + d)/(3\alpha+ d)}{ n}^{2\alpha/(3\alpha + d)}, 
\end{align*}
we have with probability at least $1-\delta$ that 
$\sup_{X_i \in \Xspl}\abs{f_n(X_i) - f(X_i)}\leq \epsilon.$
\end{lemma}

The next lemma bounds $r_{k, n}(X_i)$ in terms of $r_{k}(X_i)$, and hence, in terms of the density at $X_i$. The proof 
is provided as supplement in the appendix. 

\begin{lemma}
\label{cor:rk_rk_n}
Suppose $f$ satisfies (A.1) and (A.2). Fix $\lambda>0$ and let $\mathcal{L_\lambda} \doteq \braces{x: f(x) \geq \lambda}$. 
\begin{enumerate}[(a)]
 \item Let $r \doteq \frac{1}{2}(\lambda/2L)^{1/\alpha}$. We have 
$\forall x, x' \in \real^d$, $\norm{x-x'} \leq 2r \implies \abs{f(x) - f(x')} \leq \lambda/2$. If in addition $x \in
\mathcal{L_\lambda}$, it follows that $f(x)/2\leq f(x') \leq 2f(x)$.
\item Suppose 
$k \leq {2^{-(d+3)}} v_d(2L)^{-d/\alpha} \lambda^{(d + \alpha)/\alpha} n$. We have 
$$\forall x \in \mathcal{L}_\lambda, \, r_k(x) \leq \min\braces{2^{-3/d} r, \paren{\frac{2k}{v_d n f(x)}}^{1/d}}.$$
For $\delta>0$, if in addition $k \geq 192\ln (2n/\delta)$, we have with probability at least $1-\delta$ that for all $X_i \in
\Xspl \cap \mathcal{L_\lambda}$
\begin{align*}
2^{-3/d} r_k(X_i) \leq r_{k, n}(X_i) \leq 2^{3/d} r_k(X_i).
\end{align*}
\end{enumerate}

\end{lemma}

The main separation lemma is next. It says that if $A$ and $A'$ are separated by a sufficiently large low density region, then they 
remain separated in the empirical tree. 

\begin{lemma}[Separation]
\label{lem:separation}
Suppose $f$ satisfies (A.1) and (A.2). Let $G_n$ be the $k$-NN or mutual $k$-NN graph. 
Define $\epsilon_k \doteq 11F\sqrt{\ln(2n/\delta)/k}$, and let $\delta>0$.
% $$ 192\ln (n/\delta) \leq k \leq  {2^{-\frac{2(d+3)\alpha}{d+ 3\alpha}}} v_d^{\frac{2\alpha}{d+ 3\alpha}}(2L)^{-\frac{2d}{d+
% 3\alpha}}\cdot 
% \paren{F\sqrt{\ln (2n/\delta)}}^{\frac{2(d +\alpha)}{d + 3\alpha}}\cdot n^{\frac{2\alpha}{d + 3\alpha}}.$$
 There exists $C = C(\dist)$ such that, for
\begin{align*}
&192\ln (2n/\delta) \leq k \\
& \leq  C\paren{F\sqrt{\ln (n/\delta)}}^{{2(\alpha + d)}/{(3\alpha + d)}} n^{{2\alpha}/{(3\alpha + d)}},
\end{align*}
the following holds with probability at least $1- 2\delta$ simultaneously for any two disjoint subsets 
$A, A'$ of $\real^d$.

Let $\lambda = \inf_{x\in A\cup A'}f(x)$. 
If $A$ and $A'$ are $(\epsilon, r)$-separated for $\epsilon = 6\epsilon_k + 2\tilde{\epsilon}$ and 
$r=\frac{\theta}{2}\paren{{4k}/{v_d n \lambda}}^{1/d}$, then
$A\cap \Xspl$ and $A'\cap \Xspl$ are disconnected in 
$G_n(\lambda - 2\epsilon_k -\tilde{\epsilon})$ and therefore in $\G_n(\lambda -2\epsilon_k)$. 
\end{lemma}

\begin{proof}
Applying Lemma \ref{cor:knnbound}, it's immediate that, with probability at least $1-\delta$, all points of any $A\cup A'\cap \Xspl$  
are in $G_n(\lambda - \epsilon_k)$ and lower levels, and no point 
from $S_{+r}\cap \Xspl$ is in $G_n(\lambda - 5\epsilon_k-2\tilde{\epsilon})$ or
higher levels.
Thus any path between $A$ and $A'$ in $G_n(\lambda - 2\epsilon_k -\tilde{\epsilon})$ must have an edge through the center 
$x\in S$ of a ball $B(x, r)\subset S_{+r}$. This edge must therefore have length greater than $2r$. We just need to show
that no such edge exists in $G_n(\lambda - 2\epsilon_k -\tilde{\epsilon})$.

Let $V$ be the set of points (vertices) in  $G_n(\lambda - 2\epsilon_k -\tilde{\epsilon})$. By Lemma \ref{cor:knnbound}, 
$\min_{X_i \in V} f(X_i) \geq \lambda - 3\epsilon_k -\tilde{\epsilon}$.
Given the density assumption on $S$, $\lambda \geq 6\epsilon_k + 2\tilde{\epsilon}$ so 
$\min_{X_i \in V} f(X_i)\geq \lambda/2$ and $V \subset \mathcal{L}_{\epsilon_k}$. 
Now, given the range of $k$, Lemma \ref{cor:rk_rk_n} holds for the level set $\mathcal{L}_{\epsilon_k}$.
It follows that with
probability at least $1-\delta$ (uniform over any such choice of $A, A'$ since the event is a function of $\mathcal{L}_{\epsilon_k}$), 
$$\max_{X_i \in V} r_{k,n} (X_i) \leq 2^{3/d} \max_{X_i \in V}r_k(X_i) \leq \frac{2r}{\theta}.$$
Thus, edge lengths in $G_n(\lambda - 2\epsilon_k -\tilde{\epsilon})$ are at most $2 r$.
\end{proof}

\subsubsection{Identifying Modes}
As a corollary to Lemma \ref{lem:separation}, we can guarantee in Lemma \ref{lem:modes} that certain salient modes are recovered by the 
empirical cluster tree. 
% Lemma \ref{lem:separation} guarantees that points from well separated sets remain separated in 
% the empirical tree, however, for mode identification we also need to ensure that points are indeed sampled near
% important modes. 
For this to happen, we require in Definition \ref{def:modes} (\rm ii) that an $(\epsilon, r)$-salient mode 
$A$ is contained in a sufficiently large set $A_k$ so that we sample points near the mode.

We start with the following VC lemma establishing conditions under which subsets of $\real^d$ contain samples from $\Xspl$.
\begin{lemma}[Lemma 5.1 of \cite{BBL103}]
Suppose $\C$ is a class of subsets of $\real^d$. Let $\Sp_\C(2n)$ denote the $2n$-shatter coefficient of $\C$. Let $\dist_n$ denote the
empirical 
distribution over $n$ samples drawn i.i.d from $\dist$. For $\delta>0$, with probability at least $1-\delta$, 
\begin{align*}
\sup_{A\in \C} \frac{\dist(A) - \dist_n(A)}{\sqrt{\dist(A)}} \leq 2 \sqrt{\frac{\log \Sp_\C(2n) + \log 4/\delta}{n}}.
\end{align*}
\label{lem:VC}
\end{lemma}

\begin{lemma}[Modes]
\label{lem:modes}
Suppose $f$ satisfies (A.1) and (A.2). Let $G_n$ be the $k$-NN or mutual $k$-NN graph. 
%Define $\epsilon_k \doteq
%11F\sqrt{\ln(2n/\delta)/k}$ and $r\doteq\frac{1}{2}\paren{\epsilon_k/2L}^{1/\alpha}$.
Let $\delta>0$. 
There exist $C$ and $C' = C'(\dist)$ such that, for
\begin{align*}
 &C  d \ln (n/\delta)\\
& \leq k \leq  C'  \paren{F\sqrt{\ln (n/\delta)}}^{{2(\alpha +
d)}/{(3\alpha +
d)}} n^{{2\alpha}/{(3\alpha + d)}} 
\end{align*}
the following holds with probability at least $1- 3\delta$.
Let $\epsilon = 6\epsilon_k + 2\tilde{\epsilon}$ and 
$r=\frac{\theta}{2}\paren{{4k}/{v_d n \lambda}}^{1/d}$. There exists a $1-1$ map 
from the set of $(\epsilon, r)$-salient modes to the leaves of the empirical tree
$\braces{\G_n(\lambda)}_{\lambda >0}$.
\end{lemma}
\begin{proof}
First, with probability at least $1-\delta$, for any $(\epsilon, r)$-salient mode $A$, there are samples
 in $\Xspl$ from the containing set $A_k$ (as defined in Definition \ref{def:modes}). To arrive at this 
we apply Lemma \ref{lem:VC} for the class 
$\C$ of all possible balls $B\in \real^d$,
(for this class $\Sp_\C(2n)\leq (2n)^{d+1}$). We have with probability at least $1-\delta$ 
that for all $B$, $\dist_n(B) > 0$ whenever 
$$\dist(B) \geq \frac{C d\ln(n/\delta)}{n} > 4\frac{(d + 1)\log (2n) + \log (4/\delta)}{n},$$ 
where $C$ is appropriately chosen to satisfy the last inequality. Now, from the definition of $A_k$, 
there exists $x$ such that $B(x, r_k(x))\subset A_k$, while we have 
$\dist(B(x, r_k(x))) = k/n \geq {C d\ln(n/\delta)}/{n}$, implying that 
$\dist_n(A_k)\geq \dist_n(B(x, r_k(x))) \geq 1/n$.

As a consequence of the above argument, there is a finite number $m$ of $(\epsilon, r)$-salient modes since each 
contributes some points to the final sample $\Xspl$.
We can therefore arrange them as $\braces{A^i}_{i=1}^m$ so that for $i < j$, we have $\lambda_i \leq \lambda_j$ where 
$\lambda_i = \inf_{x\in A^i_k}f(x)$.  An injective map can now be constructed iteratively as follows. 

Starting with $i=1$, we have by Lemma \ref{lem:separation} that, 
with probability at least $1-2\delta$, $A^i_k \cap \Xspl$ is disconnected in $\G_n(\lambda_i - 2\epsilon_k)$ 
from all $A^j_k, j > i$. Let $U$ be the union of those CCs of $\G_n(\lambda_i - 2\epsilon_k)$ containing 
points from $A^i_k \cap \Xspl$. We've already established that $U$ contains no point from any $A^j_k, j > i$. 
For $i>1$, $U$ also contains no point from any $A^j_k, j < i$. This is because, again by Lemma \ref{lem:separation}, 
 $A^j_k\cap \Xspl$ is disconnected in $\G_n(\lambda_j - 2\epsilon_k)$ from $A^i_k \cap \Xspl$, therefore disconnected 
from $U$ since all CCs in $U$ remain connected at lower levels. 
Now, since $U$ is disconnected from all $A^j_k, j \neq i$, we can just map $A^i$ to any leaf rooted in $U$, 
$A^i$ being the unique image of such a leaf.
\end{proof}

\subsection{Maintaining Connectedness}
\label{sec:connectedness}
In this section we show that sample points from a connected subset $A$ of $\real^d$ remain connected in the empirical cluster tree before 
pruning (therefore also after pruning).

% For any two points $x, x' \in A\cap \Xspl$, the idea, as in \cite{CD101}, is to uncover a path in $G_n$ near a path $P$ in $A$ that connects the two. 
% The path $P$ is covered by balls, and $x, x'$ are connected in $G_n$ through a sequence $x_1 = x, x_2, \ldots x_i = x'$ of sample points in 
% these balls (the dashed path depicted below). 
Similar to \cite{CD101}, for any two points $x, x' \in A\cap \Xspl$ we uncover a path in $G_n$ near 
a path $P$ in $A$ that connects the two. 
The path in $G_n$  (the dashed path depicted below) consists of a sequence $x_1 = x, x_2, \ldots, x_i = x'$ 
of sample points from balls centered on the path $P$ in $A$ (the solid path depicted below). The intuition is that $P$
is a high density route near which we can find enough sample points to connect $x$ and $x'$.
\begin{center}
\resizebox{5.5cm}{!}{\begin{picture}(0,0)%
\includegraphics{connect2.pstex}%
\end{picture}%
\setlength{\unitlength}{4144sp}%
\begingroup\makeatletter\ifx\SetFigFont\undefined%
\gdef\SetFigFont#1#2#3#4#5{%
  \reset@font\fontsize{#1}{#2pt}%
  \fontfamily{#3}\fontseries{#4}\fontshape{#5}%
  \selectfont}%
\fi\endgroup%
\begin{picture}(6105,2102)(1786,-4988)
\put(1801,-4471){\makebox(0,0)[lb]{\smash{{\SetFigFont{30}{24.0}{\familydefault}{\mddefault}{\updefault}{\color[rgb]{0,0,0}$x$}%
}}}}
\put(7876,-3121){\makebox(0,0)[lb]{\smash{{\SetFigFont{30}{24.0}{\familydefault}{\mddefault}{\updefault}{\color[rgb]{0,0,0}$x'$}%
}}}}
\end{picture}%
}
\end{center}

The balls centered on $P$ must be chosen sufficiently small and consecutively close so that consecutive terms $x_i, x_{i+1}$ are adjacent in $G_n$. In \cite{CD101}, points 
are adjacent (at any particular level) whenever they are less than some scale $r$ apart; one can therefore
choose balls of the same radius $o(r)$ and consecutively $o(r)$ close. In our particular case, no single scale determines adjacency.
Adjacency is determined by the various nearest-neighbor radii and this creates a multiscale effect that complicates the analysis. 
One way to handle (and effectively get rid of) this multiscale effect is to choose balls on $P$ of the same radius $r$ corresponding to the 
smallest possible nearest-neighbor radius in $G_n$ (restricted to $A\cap \Xspl$). However, in order to get samples in such small balls one would need 
rather large sample size $n$, so the idea results in weak bounds. We instead use an inductive argument which keeps track of the various scales, the intuition 
being that nearest-neighbor-radii have to change slowly along the path $P$ from $x$ to $x'$. 
% This argument is described below.
% 
% Suppose we've built the sequence up to some $x_i$ so far. We will then pick $x_{i+1}$ in a ball of radius $o(r_{k, n}(x_i))$ centered at some $y_i$
% on the path $P$ in $A$, where $y_i$ is also $o(r_{k, n}(x_i))$ close to $x_i$. By maintaining the invariant that $x_i$ is close to the path $P$, we know such 
% a $y_i$ must exist. On the other hand, the existence of a sample $x_{i+1}$ in the ball $B(y_i, o(r_{k, n}(x_i)))$ is predicated on the ball having enough mass, 
% and we argue that it has mass 
% $\Omega(k/n)$ approximately the same as $B(x_i, o(r_{k, n}(x_i)))$ since $y_i$ and $x_i$ are close. 
% Such a mass of $\Omega(k/n)$ guarantees the existence of $x_{i+1}$ provided $k$ is sufficiently large. Last, we argue that 
% $r_{k, n}(x_i) \approx r_{k, n}(x_{i+1})$. This serves two purposes: first to show that $x_i$ and $x_{i+1}$ are (mutual) nearest neighbors and thus
% adjacent in $G_n$, second to maintain the invariant that $x_{i+1}$ is $o(r_{k, n}(x_{i+1})) = o(r_{k, n}(x_i))$ close to the path $P$.

%We now present the main connectedness lemma built on the intuition discussed above.

\begin{lemma}[Connectedness]
\label{lem:connectedness}
 Suppose $f$ satisfies (A.1) and (A.2). Let $G_n$ be the $k$-NN or mutual $k$-NN graph.  Define $\epsilon_k \doteq
11F\sqrt{\ln(2n/\delta)/k}$ and let $\delta>0$. 
There exist $C$ and $C' = C'(\dist)$ such that, for
\begin{align*}
 &C  \paren{\max\braces{1, {\sqrt{2}}/{\theta}}}^d d \ln (n/\delta)\\
& \leq k \leq  C'  \paren{F\sqrt{\ln (n/\delta)}}^{{2(\alpha +
d)}/{(3\alpha +
d)}} n^{{2\alpha}/{(3\alpha + d)}},
\end{align*}
the following holds with probability at least $1- 3\delta$ simultaneously for all connected subsets $A$ of $\real^d$. 

Let $\lambda \doteq \inf_{x\in A} f(x) > 2 \epsilon_k$. All points in $A\cap \Xspl$ belong to the same CC of $G_n(\lambda -
2\epsilon_k)$, therefore of $\G_n(\lambda - 2\epsilon_k)$.
\end{lemma}
\begin{proof}
First, let $C$ and $C'$ be large enough for lemmas \ref{cor:knnbound} and \ref{cor:rk_rk_n} to hold.
Define $r\doteq\frac{1}{2}\paren{\epsilon_k/2L}^{1/\alpha}$. 
By Lemma \ref{cor:rk_rk_n} (a), we have  that
 $f(x)\geq \lambda -\epsilon_k/2$ for any $x\in A_{+r}$. Applying Lemma \ref{cor:knnbound}, it follows that with probability 
at least $1-\delta$ (uniform over choices
of $A$), all points of $A_{+r} \cap \Xspl$ are in $G_n(\lambda - 2\epsilon_k)$. We will show that $A\cap \Xspl$ is connected in
$G_n(\lambda - 2\epsilon_k)$ possibly through points in $A_{+r}\setminus A$. 

In particular, any $x, x' \in A\cap \Xspl$ are connected through a sequence $\braces {x_i}_{i>1}, x_i \in A_{+r}\cap \Xspl$ built according to the following
procedure. Let $P$ be a path in $A$ between $x$ and $x'$. Define $\tau \doteq \min\braces{1, \theta/\sqrt{2}}$.
\begin{quote}
 Starting at $i = 1$ ($x_1 = x$), set $x_{i+1}=x'$ if $\norm{x_i - x'} \leq \theta \min \braces{r_{k,n}(x_i), r_{k,n}(x')}$,  and 
we're done, otherwise:\\
 Let $y_i$ be the point in $P\cap B\paren{x_i, \tau 2^{-9/d} r_{k, n}(x_i)}$ farthest along the path $P$ from $x$, i.e. $P^{-1}(y_i)$ is
highest in the set. 
Define the half-ball
\begin{align*}
 H(y_i) \doteq \{z: \norm{z - y} < \tau 2^{-18/d} r_{k, n}(x_i), \\
(z- y_i)\cdot (x_i - y_i)\geq 0\}.
\end{align*}

Pick $x_{i+1}$ in  $H(y_i)\cap \Xspl$, and continue.
\end{quote}

The rest of the argument will proceed inductively as follows. First, assume that $x_i \in A_{+r}$ and that $y_i$ exists. This is necessarily the 
case for $x_1, y_1$. Assume $x_{i+1}\neq x'$. We will show that $x_{i+1}$ exists, is also in $A_{+r}$, and is adjacent to $x_i$ in $G_n$. 
It will follow that $y_{i+1}$ must exist (if the process does not end) and is distinct from $y_1, \ldots, y_i$. We'll then argue that the process must also end.

To see that $x_{i+1}$ exists (under the aforementioned assumptions), we apply Lemma \ref{lem:VC} for the class 
$\C$ of all possible half-balls $H(y)$ centered at $y\in \real^d$ 
(for this class $\Sp_\C(2n)\leq (2n)^{2d+1}$). We have with probability at least $1-\delta$ that for all $H(y)$, $\dist_n(H(y)) > 0$
whenever 
$$\dist(H(y)) \geq \frac{C_0d\ln(\frac{n}{\delta})}{n} > \frac{(8d + 4)\log (2n) + 4\log (\frac{4}{\delta})}{n},$$ 
where $C_0$ is appropriately chosen to satisfy the last inequality. We next show $\dist(H(y_i))$ satisfies the first inequality.

We first apply Lemma \ref{cor:rk_rk_n} on $\mathcal{L}_{\epsilon_k} \supset A_{+r}$ (this inclusion was established earlier).
We have with probability at least $1-\delta$ (uniform
over all $A$) that for $x_i\in A_{+r}$, $r_{k, n}(x_i) \leq 2^{3/d} r_k(x_i) \leq  r$.  Thus, for all $z\in H(y_i)$, 
\begin{align}
 \norm{z-x_i}&\leq 2\cdot \tau 2^{-9/d} r_{k,n}(x_i)\nonumber \\
&\leq  2\cdot \tau 2^{-9/d} r \leq 2r, \label{eq:z_x_i}
\end{align}
implying by the same Lemma  \ref{cor:rk_rk_n} that $f(z) \geq f(x_i)/2$. Now, from Lemma \ref{cor:knnbound}, 
$f_n(x_i)\leq f(x_i) + \epsilon_k \leq 2f(x_i)$. We can thus write
\begin{align*}
 \dist(H(y_i)) &\geq \frac{1}{4} \vol{B(y_i, \tau2^{-18/d}r_{k,n}(x_i))}f(x_i)\\
&= \tau^d 2^{-20} \vol{B(x_i, r_{k,n}(x_i))}f(x_i)\\
&\geq \tau^d 2^{-21} \vol{B(x_i, r_{k,n}(x_i))} f_n(x_i)\\
&= \tau ^d 2^{-21} \frac{k}{n} \geq \frac{C_0d\ln(n/\delta)}{n}, \text{ for } C\geq 2^{21} C_0.
\end{align*}
Therefore there is a point $x_{i+1}$ in $H(y_i)\cap \Xspl$. In addition $x_{i+1}\in A_{+r}$ since it is within $r$ of 
$y_i\in A$. 

Next we establish that there is an edge between $x_i$ and $x_{i+1}$ in $G_n$. To this end we relate $r_{k, n}(x_{i+1})$ to $r_{k, n}(x_{i})$
by first relating $r_{k}(x_{i+1})$ to $r_{k}(x_{i})$. Remember that for $z\in
A_{+r}$ we have $r_{k}(z) < r$ so that for any $z'\in B(z, r_{k}(z))$ we have $f(z)/2 \leq f(z')\leq 2f(z)$.
Also recall that we always have $\norm{x_i - x_{i+1}} \leq 2r$ (see (\ref{eq:z_x_i})), 
implying $f(x_{i+1}) < 2 f(x_{i})$. We then have  
\begin{align*}
 v_d r_k^d(x_{i})\cdot \frac{1}{2}f(x_i) &\leq \frac{k}{n} \leq v_d r_k^d(x_{i+1})\cdot {2}f(x_{i+1}) \\
&\leq v_d r_k^d(x_{i+1})\cdot {4}f(x_{i}) ,
\end{align*}
where for the first two inequalities we used the fact that both balls $B(x_i, r_k(x_i))$ and $B(x_{i+1}, r_k(x_{i+1}))$ have the same mass $k/n$. It follows that 
\begin{align}
 r_{k, n}(x_{i+1}) &\geq 2^{-3/d} r_k(x_{i+1}) \geq 2^{-6/d} r_k(x_{i})\nonumber\\
& \geq 2^{-9/d} r_{k, n}(x_{i}), \label{eq:min_r_kn}
\end{align}
implying $2^{-9/d} r_{k, n}(x_{i}) \leq \min\braces{r_{k, n}(x_{i}), r_{k, n}(x_{i+1})}$. We then get
\begin{align*}
 \norm{x_i - x_{i+1}}^2 &= \norm{x_i - y_i}^2 + \norm{x_{i+1} - y_i}^2 \\
& - (x_i - y_i)\cdot(x_{i+1} - y_i)\\
&\leq \norm{x_i - y_i}^2 + \norm{x_{i+1} - y_i}^2 \\
&\leq 2\tau^2\cdot \min\braces{r_{k, n}^2(x_{i}), r_{k, n}^2(x_{i+1})} \\
&\leq \theta^2 \min\braces{r_{k, n}^2(x_{i}), r_{k, n}^2(x_{i+1})}, 
\end{align*}
meaning $x_i$ and $x_{i+1}$ are adjacent in $G_n$.

Finally we argue that $y_{i+1}$ must exist. By (\ref{eq:min_r_kn}) above we have
$$\norm{x_{i+1} - y_{i}} < \tau2^{-18/d}r_{k, n}(x_{i}) \leq \tau2^{-9/d}r_{k,n}(x_{i+1}),$$
in other words the ball $B\paren{x_{i+1}, \tau2^{-9/d}r_{k,n}(x_{i+1})}$ contains $y_i\in P$ in its interior. It follows 
by continuity of $P$ that there is a point $y_{i+1}$ in this ball further along the path from $x_i$ than $y_{i}$. Thus, recursively all $y_i$'s must be distinct, implying 
that all $x_i$'s must be distinct. Since all $x_i$'s belong to the finite sample $\Xspl$ the process must eventually terminate.
\end{proof}

\subsubsection{Pruning of Spurious Branches}
As a corollary to Lemma \ref{lem:connectedness} we can guarantee in Lemma \ref{lem:pruning} that the pruning 
procedure will remove 
all spurious branchings, and hence, all spurious clusters. 

\begin{lemma}[Pruning]
\label{lem:pruning}
  Let $\delta>0$. 
Under the assumptions of Lemma \ref{lem:connectedness}, the following holds with probability at least $1- 3\delta$, provided 
$\tilde{\epsilon} \geq 3 \epsilon_k$. 

Consider two disjoint 
CCs  $A_n$ and $A_n'$ at the same level in $\braces{\G_n(\lambda)}_{\lambda>0}$.
Let  $V$ be the union of vertices of $A_n$ and $A_n'$, and define $\lambda \doteq \inf_{x\in V} f(x)$. 
The vertices of $A_n$ and those of $A_n'$ are in separate CCs of $G(\lambda)$.
\end{lemma}
\begin{proof}
Let $\lambda_n = \min_{x\in V}f_n(x)$ be the level in the empirical tree containing $A_n, A_n'$. 
By Lemma \ref{cor:knnbound}, $\sup_{x\in \Xspl} \abs{f_n(x) - f(x)} \leq \epsilon_k$ so $\lambda_n \leq \lambda + \epsilon_k$. Thus, we must
have 
$\lambda > 2\epsilon_k$, since otherwise $\lambda_n \leq \tilde{\epsilon}$ implying $\G_n(\lambda_n)$ must have a single connected
component.
 
Now suppose points in $V$ were in the same component $A$ of $G(\lambda)$. 
% First, note that Theorem \ref{theo:main} (a) holds
% independently 
% of the setting of $\tilde{\epsilon}$, in particular it holds for $\tilde{\epsilon} = 0$. In other words 
By Lemma \ref{lem:connectedness}, all of $A\cap\Xspl$ is connected in $G_n(\lambda - 2\epsilon_k)$ and 
at lower levels. By the last argument $\lambda_n - \tilde{\epsilon} \leq \lambda - 2\epsilon_k$ so the pruning 
procedure reconnects $A_n$ and  $A_n'$.
\end{proof}

\section*{Acknowledgements}
We thank Sanjoy Dasgupta 
for interesting discussions which helped improve presentation.

\begin{small}
\setlength{\bibsep}{3pt}
 \bibliography{refs}
\end{small}

\bibliographystyle{icml2011}

\appendix
\section*{Appendix}
\section{Proof of Lemma \ref{cor:knnbound}}
Lemma \ref{cor:knnbound} follows as a corollary to Lemma \ref{lem:knnbound} below.

We'll often make use of the following form of the Chernoff bound.
\begin{lemma}[\cite{AV102}]
 Let $N\sim \text{Bin}(n, p)$. Then for all $0<t\leq 1$,
\begin{align*}
  \pr{N > (1+ t) np} \leq \exp\paren{-t^2 n p /3}, \\
\pr{N < (1- t) np} \leq \exp\paren{-t^2 n p /3}.
\end{align*}
\label{lem:chernoff}
\end{lemma}

\begin{lemma}
Suppose the density function $f$ satisfies:
\begin{enumerate}[(a)]
 \item $f$ is uniformly continuous on $\real^d$. In other words, 
$ \forall \epsilon > 0, \exists c_\epsilon$ s.t. for all balls $B$ where
$\vol{B}\leq c_\epsilon$ we have $ \sup_{x, x' \in B}\abs{f(x) - f(x')} < \epsilon/2$.
\item $\exists F$, $\sup_{x\in \real^d}f(x) = F$.
\end{enumerate}

 Fix $0<\epsilon < F$, let $n \geq 2$, and $k < n$. If ${k}/{n\epsilon} \leq {c_\epsilon}/{4}$ then 
$$\pr{\sup_{X_i \in \Xspl} \abs{f(X_i) - f_n (X_i)}> \epsilon}\leq 2n
\exp\paren{- \frac{\epsilon^2 k}{120 F^2}}.$$
\label{lem:knnbound}
\end{lemma}
\begin{proof}
We'll be using the short-hand notation $B_{k,n}(x)\doteq B(x, r_{k,n}(x))$ for readability in what follows. 

We start with the simple bound:
\begin{align}
 &\pr{\sup_{X_i \in \Xspl} \abs{f(X_i) - f_n (X_i)}> \epsilon} \nonumber \\
&\leq \pr{\exists X_i \in \Xspl, \, f_n(X_i) > f(X_i) + \epsilon} +\nonumber\\ 
\,& \pr{\exists X_i \in \Xspl, \, f_n(X_i) < f(X_i) - \epsilon}\nonumber\\
&= \pr{\exists X_i \in \Xspl, \, \vol{B_{k,n}(X_i)} < \frac{k}{n(f(X_i) + \epsilon)}} +\label{eq:I}\\
&\pr{\exists X_i \in \Xspl, f(X_i) > \epsilon,  \, \vol{B_{k,n}(X_i)} > \frac{k}{n(f(X_i) - \epsilon)}}\label{eq:II}
\end{align}

We handle (\ref{eq:I}) and (\ref{eq:II}) by first fixing $i$ and conditioning on $X_i = x$. We start with (\ref{eq:I}):
\begin{align}
 &\pr{\exists X_i \in \Xspl, \, \vol{B_{k,n}(X_i)} < \frac{k}{n(f(X_i) + \epsilon)}} \nonumber \\
&\leq n \int_{x} \pr{\vol{B_{k,n}(x)} < \frac{k}{n(f(x) + \epsilon)}}\, d\dist(x), \label{eq:int} 
\end{align}
where the inner probability is over the choice of $\Xspl\setminus\braces{X_i = x}$ for $i$ fixed. In what follows we use the
notation $\dist_{n-1}$ to denote the empirical distribution over $\Xspl\setminus\braces{X_i = x}$.

Assume $\vol{B_{k,n}(x)} < k/n(f(x) + \epsilon) < k/ n \epsilon < c_\epsilon$. Then by the uniform continuity assumption on $f$ we have
\begin{align*}
 &\dist\paren{B_{k,n}(x)} < \paren{f(x) + \epsilon/2}\frac{k}{n(f(x) + \epsilon)} \\
&= \paren{1 - \frac{\epsilon}{2(f(x) +\epsilon)}}\frac{k}{n}
\leq \paren{1 - \frac{\epsilon}{4F}}\frac{k}{n}
\end{align*}
Now let $B(x)$ be the ball centered at $x$ with $\dist$-mass $\paren{1 - {\epsilon}/{4F}}({k}/{n})$. Since by the
above, $\dist\paren{B_{k,n}(x)} < \dist\paren{B(x)}$, we also have that $\dist_n\paren{B_{k,n}(x)} < \dist_n\paren{B(x)}$. This implies
that 
\begin{align*}
 \dist\paren{B(x)} &< \paren{1-\frac{\epsilon}{4F}}\frac{k}{n-1} =  \paren{1-\frac{\epsilon}{4F}} \dist_n \paren{B_{k,n}(x)} \\
&\leq  \paren{1-\frac{\epsilon}{4F}} \dist_n \paren{B(x)}.
\end{align*}

In other words, let $t = \epsilon/(4F-\epsilon)$, applying the Chernoff bound of Lemma \ref{lem:chernoff}, we have   
\begin{align*}
 &\pr{\vol{B_{k,n}(x)} < k/n(f(x) + \epsilon)} \\
&\leq \pr{\dist_n \paren{B(x)} > (1+t) \dist\paren{B(x)}}\\
&\leq \exp\paren{-t^2 (n-1) \dist\paren{B(x)}/3}\leq \exp\paren{-\epsilon^2 k / 96 F^2}.
\end{align*}
Combine with (\ref{eq:int}) to complete the bound on (\ref{eq:I}).

We now turn to bounding (\ref{eq:II}). We proceed as before by fixing $i$ and integrating over $X_i = x$ where $f(x) > \epsilon$, that is 
\begin{align}
 &\pr{\exists X_i \in \Xspl, f(X_i) > \epsilon,  \, \vol{B_{k,n}(X_i)} > \frac{k}{n(f(X_i) - \epsilon)}} \nonumber\\
&\leq n \displaystyle\int_{x, f(x) > \epsilon} \pr{\vol{B_{k,n}(x)} > \frac{k}{n(f(x) - \epsilon)}}\, d \dist(x), \label{eq:int2}
\end{align}
where again the probability is over the choice of $\Xspl \setminus \braces{X_i=x}$.
Now, we can no longer infer how much $f$ deviates within $B_{k,n}(x)$ from just the event in question (as we did for the other direction).
The trick (inspired by \cite{DW104}) is to consider a related ball. 

Let $B(x)$ be the ball centered at $x$ of volume $k/n(f(x) - 3\epsilon/4)$. Then 
\begin{align*}
 \vol{B_{k,n}(x)} > \frac{k}{n(f(x) - \epsilon)} > \vol{B(x)}\\
\implies \dist_{n-1}\paren{B(x)} \leq \frac{k-1}{n-1} < \frac{k}{n}.
\end{align*}
Since $\vol{B(x)} < {4k}/{\epsilon} < c_\epsilon$, we have by the uniform continuity of $f$ that 
\begin{align*}
\dist\paren{B(x)} &>  \frac{k (f(x) - \epsilon/2)}{n (f(x) - 3\epsilon/4)} >\paren{1+ \frac{\epsilon}{4F}} \frac{k}{n} \\
&> \paren{1+ \frac{\epsilon}{4F}}\dist_{n-1}\paren{B(x)}.
\end{align*}
we thus have for $t = \epsilon/(4F + \epsilon)$, and using Lemma \ref{lem:chernoff} that 
\begin{align*}
&\pr{\vol{B_{k,n}(x)} > \frac{k}{n(f(x) - \epsilon)}} \\
&\leq \pr{\dist_n\paren{B(x)}\leq (1-t) \dist\paren{B(x)}}\\
&\leq \exp\paren{-t^2 (n-1)\dist\paren{B(x)}} \\
&\leq \exp\paren{-\epsilon^2 k/120F^2 }.
\end{align*}
Combine with (\ref{eq:int2}) to complete the bound on (\ref{eq:II}).

The final result is proved by then combining the bounds on (\ref{eq:I}) and (\ref{eq:II}). 
\end{proof}

\begin{proof}[Proof of Lemma \ref{cor:knnbound}]
For any $0<\epsilon< 1$, let $c_\epsilon = v_d 2^{-d}\paren{{\epsilon} /{2L}}^{d/\alpha}$ so that whenever for balls $B$,
$\vol{B} < c_\epsilon$, the radius $r$ of $B$ is less than $\frac{1}{2}\paren{{\epsilon}/{2L}}^{1/\alpha}$.  Thus, 
$\sup_{x, x' \in B}\abs{f(x) - f(x')}\leq L(2r)^\alpha < \epsilon/2$. Now, 
for the settings of $\epsilon$ and $k$ in the lemma statement, we have
\begin{align*}
 0<\epsilon< F \text{ and }  \frac{4k}{n\epsilon} < v_d 2^{-d}\paren{\frac{\epsilon} {2L}}^{d/\alpha} = c_\epsilon,
\end{align*}
so we can apply Lemma \ref{lem:knnbound} to get
\begin{align*}
 \pr{\sup_{X_i \in \Xspl} \abs{f(X_i) - f_n (X_i)}> \epsilon}&\leq 2n\exp\paren{- \frac{\epsilon^2 k}{120 F^2}} \\
&< \delta.
\end{align*}
\end{proof}

\section{Proof of Lemma \ref{cor:rk_rk_n}}
Lemma \ref{cor:rk_rk_n} follows as a corollary to Lemma \ref{lem:rk_rk_n} below.

\begin{lemma}
\label{lem:rk_rk_n}
Consider a subset $A$ of $\real^d$ such that  there exists $r$, satisfying
$$\forall x \in A, \, \norm{x-x'} < 2r \implies \frac{1}{2}f(x) \leq f(x') \leq 2f(x).$$
Assume $X_i \in \Xspl\cap A$. We have 
\begin{align*}
&\pr{ r_{k, n}(X_i) \geq 2^{3/d} r_k(X_i) \,|\, r_k(X_i) < 2^{-3/d}r} \\
&\leq \exp\paren{-k/12},\\
&\pr{ r_{k, n}(X_i) \leq 2^{-3/d} r_k(X_i) \, | \,r_k(X_i) < 2^{-3/d}r}\\
&\leq \exp\paren{-k/192}.
\end{align*}
\end{lemma}
\begin{proof}
Let $X_i \in \Xspl$, and fix $X_i = x\in A$ such that $r_k(x) < 2^{-3/d}r$. We automatically have
\begin{align*}
 \frac{1}{2}\vol{B(x, r_k(x))}f(x) &\leq \dist\paren{B(x, r_k(x))}\\
& \leq {2}\vol{B(x, r_k(x))}f(x).
\end{align*}
We similarly have 
\begin{align*}
 \dist\paren{B(x, 2^{3/d} r_k(x))} &\geq \vol{B(x, 2^{3/d} r_k(x))} \frac{f(x)}{2} \\
&\geq 8 \vol{B(x, r_k(x))} \frac{f(x)}{2} \\
&\geq 2\dist\paren{B(x, r_k(x))} = 2\frac{k}{n}.
\end{align*}
Again, similarly
\begin{align*}
 \frac{k}{32n}&= \frac{1}{32}\dist\paren{B(x, r_k(x))} \leq \dist\paren{B(x, 2^{-3/d} r_k(x))} \\
&\leq \frac{1}{2}\dist\paren{B(x, r_k(x))} =
\frac{k}{2n}.
\end{align*}
Thus by Lemma \ref{lem:chernoff}, 
\begin{align*}
&\pr{r_{k,n}(x) > 2^{3/d} r_k (x)} \leq \\
& \pr {\dist_{n-1} \paren{B(x, 2^{3/d} r_k (x))} < \frac{k}{n}
 \leq \frac{1}{2}\dist\paren{B(x, 2^{3/d}r_k(x))}}\\
&\leq  \exp\paren{-(n-1) \dist\paren{B(x, 2^{3/d}
r_k(x))}/12} \leq \exp\paren{-k/12}, 
\end{align*}
and 
\begin{align*}
 &\pr{r_{k,n}(x) < 2^{-3/d} r_k (x)} \leq \\
& \pr {\dist_{n-1} \paren{B(x, 2^{-3/d} r_k (x))} > \frac{k}{n} \geq 2\dist\paren{B(x, 2^{-3/d}
r_k(x))}}\\
&\leq  \exp\paren{-(n-1) \dist\paren{B(x, 2^{-3/d}
r_k(x))}/3} \leq \exp\paren{-k/192}.
\end{align*}
Conclude by integrating these probabilities over possible values of $X_i = x\in A$.
\end{proof}

\begin{proof}[Proof of Lemma \ref{cor:rk_rk_n}]
Part (a) follows directly from the Holder assumption on $f$. For part (b), notice that 
\begin{align*}
\sup_{x\in \mathcal{L}_\lambda} v_d r_k^d(x) \lambda \leq \inf_{x\in \mathcal{L}_\lambda}\dist\paren{B(x, r_k(x))} 
= \frac{k}{n}
\end{align*}
so that $\sup_{x\in \mathcal{L}_\lambda} r_k(x) \leq 2^{-3/d}r$ for the setting of $k$. 
Now using part (a) again we have for all $x\in  \mathcal{L}_\lambda$ 
\begin{align*}
 v_d r_k^d(x) \cdot\frac{f(x)}{2} \leq \dist\paren{B(x, r_k(x))} = \frac{k}{n},
\end{align*}
so $r_k(x) \leq (2k/v_d n f(x))^{1/d}$.

Finally, the probabilistic statement is obtained by applying Lemma \ref{lem:rk_rk_n} and a union-bound over $\Xspl\cap\mathcal{L}_\lambda$.
\end{proof}

\end{document}